\documentclass[letterpaper, 10 pt, conference]{ieeeconf}
\IEEEoverridecommandlockouts
\usepackage{amsmath,amssymb}
\usepackage{mathtools} 

\DeclarePairedDelimiter\ceil{\lceil}{\rceil}
\DeclarePairedDelimiter\floor{\lfloor}{\rfloor}

\newcommand{\defeq}{\triangleq}

\newcommand{\sbr}[1]{\left[ #1 \right]}
\newcommand{\sth}[2]{\left\{\ #1 \;\middle|\; #2 \ \right\}}

\newcommand{\bo}[1]{\mathcal{O} \left( #1 \right)}
\newcommand{\tbo}[1]{\tilde{\mathcal{O}} \left( #1 \right)}
\newcommand{\pmin}{\rho}
\newcommand{\pr}[1]{\mathbb{P} \left( #1 \right)}
\newcommand{\A}{\mathcal{A}}
\newcommand{\B}{\mathfrak{B}}
\newcommand{\e}{\mathfrak{E}}

\newcommand{\cE}[2]{\mathbb{E} \left[ #1 \;\middle\vert\; #2 \right]}
\newcommand{\I}{\mathcal{I}}
\newcommand{\T}{\mathcal{T}}
\newcommand{\Z}{\mathbb{Z}_+}

\usepackage{algorithm, algpseudocode}
\usepackage{setspace}
\let\Algorithm\algorithm
\renewcommand\algorithm[1][]{\Algorithm[#1]\setstretch{1.2}}

\algnewcommand\algorithmicinputs{\textbf{Inputs:}}
\algnewcommand\INPUTS{\item[\algorithmicinputs]}
\algnewcommand\algorithmicoutputs{\textbf{Outputs:}}
\algnewcommand\OUTPUTS{\item[\algorithmicoutputs]}
\algnewcommand\algorithmicinitialization{\textbf{Initialization:}}
\algnewcommand\INITIALIZATION{\item[\algorithmicinitialization]}

\usepackage{amsthm}

\newtheorem{Lemma}{Lemma}
\newtheorem{Propn}{Proposition}
\newtheorem{Thm}{Theorem}
\newtheorem{Cor}{Corollary}

\theoremstyle{definition}

\newtheorem{Aspn}{Assumption}

\theoremstyle{remark}

\usepackage{enumerate}

\usepackage{color}

\usepackage{textcomp}

\usepackage{cite}

\title{\LARGE \bf
Nonstochastic Bandits with Infinitely Many Experts
}

\author{X. Flora Meng, Tuhin Sarkar, and Munther A. Dahleh
\thanks{This work was supported by the OCP Group.}
\thanks{X. Flora Meng, Tuhin Sarkar, and Munther A. Dahleh are with the Department of Electrical Engineering and Computer Science, Massachusetts Institute of Technology, Cambridge, MA 02139, USA
        {\tt\small xmeng@mit.edu, tuhin91@gmail.com, dahleh@mit.edu}}%
}

\begin{document}

\maketitle
\thispagestyle{empty}
\pagestyle{empty}

\begin{abstract}

We study the problem of nonstochastic bandits with expert advice, extending the setting from finitely many experts to any countably infinite set: A learner aims to maximize the total reward by taking actions sequentially based on bandit feedback while benchmarking against a set of experts. We propose a variant of Exp4.P that, for finitely many experts, enables inference of correct expert rankings while preserving the order of the regret upper bound. We then incorporate the variant into a meta-algorithm that works on infinitely many experts. We prove a high-probability upper bound of $\tilde{\mathcal{O}} \big( i^*K + \sqrt{KT} \big)$ on the regret, up to polylog factors, where $i^*$ is the unknown position of the best expert, $K$ is the number of actions, and $T$ is the time horizon. We also provide an example of structured experts and discuss how to expedite learning in such case. Our meta-learning algorithm achieves optimal regret up to polylog factors when $i^* = \tilde{\mathcal{O}} \big( \sqrt{T/K} \big)$. If a prior distribution is assumed to exist for $i^*$, the probability of optimality increases with $T$, the rate of which can be fast.

\end{abstract}

\section{INTRODUCTION}
Early work on the multi-armed bandit problem commonly studied settings where the rewards of each arm are stochastically generated from some unknown distribution~\cite{robbins1952some,lai1985asymptotically,auer2002finite}. In general, such statistical assumptions are difficult to validate or inapproriate for some applications such as packet transmission in communication networks~\cite{auer1995gambling,auer2002the}. The problem of nonstochastic bandits, first investigated in~\cite{auer1995gambling,auer2002the}, makes no statistical assumptions about how the rewards are generated.

A setting of the nonstochastic bandit problem allows for incorporating expert advice. The \emph{learner} interacts with an \emph{adversary} over a time horizon $T$ as follows. At each time, the adversary sets the rewards for $K$ actions and keeps them secret. After getting every expert's advice on the probability of choosing each action, the learner combines the advice and samples an action. Finally, the learner observes only the reward of the action chosen, and the game repeats. The learner's goal is to minimize \emph{regret}, which is the gap between the total reward gained and the expected total reward of the best expert $i^*$ who is unknown a priori.

The framework described is a general one. First, there is no assumption about the generation of rewards except that the adversary is \emph{oblivious}. In other words, the adversary's choices are independent of the learner's strategy. Equivalently, all rewards can be assigned before the game starts, and the learner only observes the rewards of chosen actions sequentially. Second, we do not restrict or assume knowledge of how the experts come up with their advice. Third, experts can give deterministic advice.

The problem of bandits with expert advice is not only a natural model for numerous real-world applications, such as selecting and pricing online advertisements~\cite{mcmahan2009tighter}, but also important from a theoretical perspective. Contextual bandits can be framed as a bandits with expert advice problem by introducing policies that map a context to a probability distribution over actions~\cite{mcmahan2009tighter,agarwal2017corralling}. Bandits with expert advice are also closely related to online model selection where experts correspond to model classes~\cite{cesa2006prediction,foster2017parameter,foster2019model}.

Prior work on nonstochastic bandits with expert advice typically assumes the number of experts to be \emph{finite}~\cite{auer1995gambling,auer2002the,mcmahan2009tighter,beygelzimer2011contextual,neu2015explore}. The \textbf{exp}onential-weight algorithm for \textbf{exp}loration and \textbf{exp}loitation using \textbf{exp}ert advice (Exp4), introduced by~\cite{auer1995gambling,auer2002the}, has a regret upper bound of $\mathcal{O} \big( \sqrt{KT\ln{N}} \big)$ \emph{in expectation}, where $N$ is the number of experts. This upper bound almost matches the lower bound $\Omega (\sqrt{(KT\ln N) / \ln K})$ derived by \cite{alekh2012contextual} for the expected regret when $\ln{N} \leq T\ln{K}$. However, Exp4 does not satisfy a similar regret guarantee \emph{with high probability} due to the large variance of its estimates. Algorithms with high-probability guarantees are preferred for domains that need reliable methods, but such algorithms require delicate analysis~\cite{beygelzimer2011contextual,neu2015explore}. The Exp4.P algorithm, a variant of Exp4 proposed by~\cite{beygelzimer2011contextual}, satisfies a regret upper bound of $\mathcal{O} \big( \sqrt{KT\ln{(N/\delta)}} \big)$ with probability at least $1-\delta$. This bound can be improved by a constant factor by avoiding explicit exploration~\cite{neu2015explore}.

We study the problem of nonstochastic bandits with infinitely many experts. Our main question is: Can the learner perform almost as well as the globally best expert $i^*$ of a countably infinite set while only querying a finite number of experts? This question is motivated by challenges encountered in practical situations where it is unfeasible to seek advice from all experts all the time~\cite{seldin2013open}. For search engine advertising, a company may need to choose among a multitude of schemes some of which also involve hyperparameter tuning~\cite{mcmahan2009tighter}. As another example, there are often a myriad of features that can be used for online recommendation systems. Some features tend to be more informative than others, but their relevance is normally unknown a priori. We can transform this problem into bandits with expert advice where each expert corresponds to a model class in a certain feature space. The number of experts can be extremely large due to the combinatorial nature. In contrast to the large number of experts available, it is desirable to query only some of them each time in consideration of computational constraints.

\emph{Our Contributions:} For the general case without any assumption about the experts, we propose an algorithm called \textbf{Be}st \textbf{E}xpert \textbf{S}earch (BEES) and provide theoretical guarantees on its performance. BEES runs a subroutine called Exp4.R in epochs, an algorithm that we obtain by modifying Exp4.P. The ``R'' denotes a feature of Exp4.R: it enables inference of correct expert rankings with high probability in addition to satisfying a regret upper bound of the same order as that proved for Exp4.P. Our main result establishes a high-probability upper bound of $\tilde{\mathcal{O}} \big( (i^*)^{1/\alpha}K + \sqrt{\alpha KT} \big)$ on the regret of BEES, hiding only polylog factors, which adapts to the index of the unknown best expert $i^*$ and depends on a positive integer-valued parameter $\alpha$. The bound illustrates the trade-off, controlled by $\alpha$, between exploration and exploitation for the problem of nonstochastic bandits with infinitely many experts. On the one hand, it is desirable to include numerous experts per epoch so as to approach $i^*$ at a fast rate. On the other hand, querying too many experts simultaneously necessitates long epochs, which reduces the rate at which more experts are included. Although tuning $\alpha$ needs the unknown index $i^*$, we can simply set $\alpha =1$. Our regret upper bound is optimal up to polylog factors when $i^* = \tilde{\mathcal{O}} \big( \sqrt{T/K} \big)$. This regime is less restricted than it seems at first sight. If we assume a prior distribution on $i^*$, then $i^* = \tilde{\mathcal{O}} \big( \sqrt{T/K} \big)$ holds with a probability that increases with $T$, the rate of which can be fast. Inspired by the problem of finite-time model selection for reinforcement learning (RL), we also present an example of structured experts, which simulates the trade-off between approximation and estimation. We discuss how the expert ranking property of Exp4.R can be used to expedite learning in such case and demonstrate the improvement in numerical experiments.

\emph{Related Work:} A natural approach is to consider experts as arms and use methods for infinitely many-armed bandits such as~\cite{berry1997bandit,kleinberg2008multi,rusmevichientong2010linearly,carpentier2015simple}. However, such work relies on statistical assumptions, whereas our setting is nonstochastic. Our question is also related to bandits with limited advice, first posed by \cite{seldin2013open} and subsequently solved by \cite{kale2014multiarmed}, but their results are restricted to finitely many experts. For the setting considered in this paper, existing work either achieves a high-probability regret bound larger than $\tilde{\mathcal{O}} \big( \sqrt{KT} \big)$ or has worse computational efficiency. When configured correctly, Exp4 has a regret upper bound of $\mathcal{O} \big( \sqrt{KT}\ln{i^*} \big)$ \emph{in expectation} \cite{foster2019model}. However, the algorithm is computationally unfeasible as it needs to handle infinitely many experts at every time step. One method of making Exp4 computationally tractable is to truncate the sequence of experts to a subset of size $\mathcal{O} \big( e^{\sqrt{KT}} \big)$ as any larger set would make the expected regret superlinear in $K$ or $T$. Running Exp4 with correct configurations on this subset of experts has a regret upper bound of $\mathcal{O} \big( (KT)^{3/4} + T\Delta \big)$ \emph{in expectation} where $\Delta$ is the infimum upper bound on the suboptimality gaps of the experts considered. For \emph{stochastic} contextual bandits, Exp4.P can be used as a subroutine to achieve a high-probability regret bound of $\tilde{\mathcal{O}} \big( \sqrt{dT\ln{T}} \big)$ with an infinite set of experts that has a finite Vapnik--Chervonenkis dimension $d$~\cite{beygelzimer2011contextual}. Since the regret analysis of Exp4.P relies on the union bound, the algorithm does not apply to infinitely many experts in the nonstochastic setting. If we run Exp4.P on a finite subset of experts of size $\Theta \big( \delta \exp \big( \sqrt{T/(16K)} \big) \big)$, the regret is then bounded from above by $\mathcal{O} \big( K^{1/4}T^{3/4} + T\Delta \big)$ with probability at least $1-\delta$. Running Exp4.P on a subset of $T$ experts attains a high-probability regret bound of $\tilde{\mathcal{O}} \big( \sqrt{KT} \big)$ when $i^* \leq T$. Although the worst-case regret guarantee is the same order as that provided by our algorithm for sufficiently small $i^*$, considering a subset of experts that is fixed in advance can lead to worse performance in practice than growing the subset adaptively, as shown in our numerical experiments. Moreover, the truncation method requires knowing $T$ a priori, which is not necessary for BEES. Since the computational complexity of Exp4.P is linear in the number of experts for both space and runtime, running Exp4.P on $T$ experts becomes computationally intensive for large $T$.

\emph{Outline:} Section~\ref{sec:problem_formulation} formally defines the problem of nonstochastic bandits with infinitely many experts. In Section~\ref{sec:finite_experts}, we introduce Exp4.R for the setting of finitely many experts and prove that it enables inference of correct expert rankings with high probability. Section~\ref{sec:infinite_experts} investigates the case of infinitely many experts and presents a meta-algorithm that runs Exp4.R as a subroutine. We prove a high-probability regret upper bound and give an example to illustrate how to expedite learning when working with structured experts. Section~\ref{sec:experiments} presents simulation results that complement our theoretical findings. Finally, we conclude in Section~\ref{sec:conclusions}.

\section{PROBLEM FORMULATION}\label{sec:problem_formulation}
Let $\Z$ be the set of strictly positive integers. For $N \in \Z$, we define $[N] \defeq \{1,2, \dots , N\}$. Let $T \in \Z$ be the \emph{time horizon}. Let $\A$ be a set of \emph{actions} where $| \A | =K < \infty$.

At each time $t \in [T]$, the adversary first sets a reward vector $r(t) \in [0,1]^K$ where $r_a(t)$ is the \emph{reward} of action $a$. Each expert $i \in \Z$ then gives their \emph{advice} $\xi^i(t)$, which is a probability vector over $\A$. After querying a finite subset of the experts' advice but not the rewards, the learner then samples an action $a(t)$. Finally, the learner receives the reward $r_{a(t)}(t)$ and no other information. The game proceeds to time $t+1$ and finishes after $T$ time steps. The learner's goal is to combine the experts' advice such that the total reward is close to a benchmark, which we will define shortly.

Let $y_i(t) \defeq \sum_{a \in \A} \xi^i_a(t) r_a(t)$ be the expected reward of expert $i$ at time $t$. For any time interval $\T \subset \Z$ such that $| \T | < \infty$, we denote the expected total reward of expert $i$ during $\T$ as $R_i(\T) \defeq \sum_{t \in \T} y_i(t)$. We define the best expert $i^*(\I;\T)$ of a subset $\I \subseteq \Z$ during $\T$ as the one with the lowest index that has the highest total reward in expectation,\footnote{If $\max_{i \in \I} R_i(\T)$ does not exist, we define $i^*(\I;\T) = \infty$ and $R_{i^*(\I;\T)}(\T) = \sup_{i \in \I} R_i(\T)$.} namely, $i^*(\I;\T) \defeq \min \left\{ \operatorname{argmax}_{i \in \I} R_i(\T) \right\}$. The learner's \emph{regret} with respect to $i^*(\I;\T)$ is
\begin{equation*}
	\text{Regret}(\T; \I) \defeq R_{i^*(\I;\T)}(\T) - \sum_{t \in \T} r_{a(t)}(t).
\end{equation*}
For simplicity of notation, let $\text{Regret}(T) \defeq \text{Regret}([T]; \Z)$ and $i^* \defeq i^*(\Z;[T])$. The learner's goal is to minimize $\text{Regret}(T)$, the regret with respect to the \emph{globally best expert} $i^*$ for the time horizon considered.

\section{NONSTOCHASTIC BANDITS WITH A FINITE NUMBER OF EXPERTS}\label{sec:finite_experts}
We start with a simplified problem where the number of experts is finite. Section~\ref{sec:exp4r_algorithm} presents Exp4.R (Algorithm~\ref{alg:exp4.r}) and provides some intuition for its design. In Section~\ref{sec:exp4r_properties}, we show that Exp4.R not only preserves the regret upper bound of Exp4.P in terms of order but also enables inference of correct expert rankings with high probability.

\subsection{The Algorithm}\label{sec:exp4r_algorithm}
Exp4.R (Algorithm~\ref{alg:exp4.r}) is a slight variant of Exp4.P proposed by~\cite{beygelzimer2011contextual}. The major distinction is that Exp4.R calculates a threshold vector $\epsilon$ which enables inference of correct expert rankings with high probability. Exp4.R takes four inputs, namely, an \emph{error rate} $\delta \in (0,1]$, a time horizon $T \in \Z$, the minimum probability $\pmin \in (0, 1/K]$ of exploration, and a finite set of experts $\I \subset \Z$. Without loss of generality, we suppose that $| \I | =N$.

Exp4.R first initializes a weight $w_i(1)=1$ for each expert $i \in \I$. At time $t \in [T]$, normalizing $w(t)$ gives a probability distribution $q(t)$ over $\I$. After getting advice $\xi^i(t)$ from each expert $i$, Exp4.R constructs a probability distribution $p(t)$ over $\A$ by weighting all advice according to $q(t)$ and mixing in uniform exploration so that $p_a(t) \geq \pmin$ for all $a \in \A$. Specifically, for all $a$, let
\begin{equation}
	p_a(t) = (1-K\pmin) \sum_{i \in \I} q_i(t) \xi^i_a(t) + \pmin .
\label{eq:exp4.r_action_prob}
\end{equation}
Exp4.R subsequently takes action $a(t)$ sampled according to $p(t)$ and receives the reward $r_{a(t)}(t)$. Time $t$ concludes with weight updates as specified below. For $i \in \I$, Exp4.R estimates $y_i(t)$ by $\hat{y}_i(t)$ and calculates an upper bound on the variance of $\hat{y}_i(t)$ conditional on history until time $t-1$ as given by
\begin{equation}
	\hat{y}_i(t) = \frac{\xi^i_{a(t)}(t) r_{a(t)}(t)}{p_{a(t)}(t)}, \quad \hat{v}_i(t) = \sum_{a \in \A} \frac{\xi^i_a(t)}{p_a(t)}.
\label{eq:exp4.r_update_yv}
\end{equation}
Exp4.R updates each expert's weight $w_i(t)$ using
\begin{equation}
	w_i(t+1) = w_i(t) \exp \left( \frac{\pmin}{2} \left[ \hat{y}_i(t) + \beta \hat{v}_i(t) \right] \right) ,
\label{eq:exp4.r_update_w}
\end{equation}
where $\beta = \sqrt{\ln (2N/\delta) / (KT)}$.
The game ends in $T$ time steps and gives two outputs, namely, the final weight vector $w(T+1)$ and a threshold vector $\epsilon$, the $i$th entry of which is
\begin{equation*}
	\epsilon_i = \left[ 1 + \frac{1}{KT} \sum_{t=1}^T \hat{v}_i(t) \right] \ln \left( \frac{2N}{\delta} \right) .
\end{equation*}

\begin{algorithm}[th!]
	\caption{Exp4.R}
	\label{alg:exp4.r}
\begin{algorithmic}
	\State {\bfseries Input:} $\delta \in (0, 1]$, $T \in \Z$, $\pmin \in (0,1/K]$, $\I \subset \Z$
	\State {\bfseries Output:} $w(T+1)$, $\epsilon$
	\State $\beta \gets \sqrt{\ln (2N/\delta) / (KT)}.$
	\State $w_i(1) \gets 1$ for $i \in \I$.
	\For{$t = 1, \dots , T$}
		\State Get $\xi^i(t)$ for $i \in \I$.
		\State $q_i(t) \gets w_i(t) / \sum_{i' \in \I} w_{i'}(t)$ for $i \in \I$.
		\State $p_a(t) \gets (1-K\pmin) \sum_{i \in \I} q_i(t) \xi^i_a(t) + \pmin$ for $a \in \A$.\label{alg:exp4.r_action_prob}
		\State Sample action $a(t)$ from $p(t)$.
		\State Take action $a(t)$ and receive reward $r_{a(t)}(t)$.
		\For{$i \in \I$}
			\State $\displaystyle
			    \begin{aligned} 
			    	\hat{y}_i(t) &\gets \frac{\xi^i_{a(t)}(t) r_{a(t)}(t)}{p_{a(t)}(t)},\\
			    	\hat{v}_i(t) &\gets \sum_{a \in \A} \frac{\xi^i_a(t)}{p_a(t)},\\
			    	w_i(t+1) &\gets w_i(t) \exp \left( \frac{\pmin}{2} \left[ \hat{y}_i(t) + \beta \hat{v}_i(t) \right] \right) .
			    \end{aligned}$ \label{alg:exp4.r_update_yvw}
		\EndFor
	\EndFor
	\For{$i \in \I$}
		\State
			\begin{equation*}
				\epsilon_i \gets \left[ 1 + \frac{1}{KT} \sum_{t=1}^T \hat{v}_i(t) \right] \ln \left( \frac{2N}{\delta} \right) .
			\end{equation*}
	\EndFor
\end{algorithmic}
\end{algorithm}

\subsection{Properties}\label{sec:exp4r_properties}
We establish in Proposition~\ref{propn:exp4r_ranking} that, with high probability, Exp4.R not only satisfies a regret upper bound of the same order as that proved for Exp4.P but also reveals correct pairwise expert rankings if the corresponding weights are sufficiently separated. We give some intuition here and provide proofs in the appendix.

For simplicity of notation, we denote $R_i([T]) \defeq \sum_{t=1}^T y_i(t)$ as $R_i(T)$. Updating weights using \eqref{eq:exp4.r_update_w} allows us to construct a confidence bound for each $R_i(T)$. For $i \in \I$, let $\hat{R}_i(T) \defeq \sum_{t=1}^T \hat{y}_i(t)$ and $\hat{V}_i(T) \defeq \sum_{t=1}^T \hat{v}_i(t)$. For any $\delta \in (0, 1]$, let $\e (\delta)$ be an event defined by
\begin{align*}
	\forall i \in \I , &\quad -\ln \left( \frac{2N}{\delta} \right) \sqrt{\frac{KT}{\ln N}} - \sqrt{\frac{\ln N}{KT}} \hat{V}_i(T) \\
	&\leq R_i(T) - \hat{R}_i(T) \\
	&\leq \sqrt{\ln \left( \frac{2N}{\delta} \right)} \left( \frac{\hat{V}_i(T)}{\sqrt{KT}} + \sqrt{KT} \right) .
\end{align*}

Lemma~\ref{lemma:pac_bound_on_total_reward_estimation_error} shows that the estimates $\hat{R}_i(T)$ are concentrated around the true values $R_i(T)$. The proof relies on a Freedman-style inequality for martingales from \cite{beygelzimer2011contextual}.

Lemma~\ref{lemma:beygelzimer11_exp4p_regret} establishes an upper bound on the regret of Exp4.R. Since Lemma~\ref{lemma:beygelzimer11_exp4p_regret} is a slight variant of Theorem~2 in \cite{beygelzimer2011contextual}, the proof is very similar to the original one and hence omitted. We note that Theorem~2 in \cite{beygelzimer2011contextual} holds for a smaller regime than stated in the original paper. To be specific, the condition $T=\Omega (K\ln N)$ is essential for $\pmin = \sqrt{\ln{N}/(KT)} \leq 1/K$ to be true. We make the correction in Lemma~\ref{lemma:beygelzimer11_exp4p_regret}.

Lemma~\ref{lemma:exp4r_ranking} validates the correctness of the inferred expert rankings when the concentration event $\e (\delta)$ holds. Corollary~\ref{cor:exp4r_ranking} shows that the uncertainty gap for ranking any pair of experts is the sum of their thresholds given by Exp4.R. We can prove Corollary~\ref{cor:exp4r_ranking} by first taking the contrapositive of the statement in Lemma~\ref{lemma:exp4r_ranking} and then switching $i$ and $i'$.

Finally, we combine the lemmas to obtain Proposition~\ref{propn:exp4r_ranking}. Same as Exp4.P, the computational complexity of Exp4.R is $\bo{KN}$ for space and $\bo{KNT}$ for runtime.

\begin{Aspn}
The following conditions hold: (i) $\max \{4K\ln N , \ln (2N/\delta)/[(e-2)K] \} \leq T$, (ii) and there exists a \emph{uniform expert} $i \in \I$ such that $\xi^i_a(t)=1/K$ for all $a \in \A$ and $t \in \Z$.
\label{assumption:regime_uniform_expert}
\end{Aspn}

\begin{Lemma}
Under Assumption~\ref{assumption:regime_uniform_expert}, if we run Exp4.R with $\pmin = \sqrt{\ln{N}/(KT)}$, then $\pr{\e (\delta)} \geq 1-\delta$ for all $\delta \in (0, 1]$.
\label{lemma:pac_bound_on_total_reward_estimation_error}
\end{Lemma}

\begin{Lemma}
Under Assumption~\ref{assumption:regime_uniform_expert}, for any $\delta \in (0, 1]$, if $\e (\delta)$ holds, then Exp4.R with $\pmin = \sqrt{\ln{N}/(KT)}$ satisfies that $\text{Regret}(T; \I) \leq 7 \sqrt{KT\ln{(2N/\delta)}}$.
\label{lemma:beygelzimer11_exp4p_regret}
\end{Lemma}

\begin{Lemma}
Under Assumption~\ref{assumption:regime_uniform_expert}, for any $\delta \in (0, 1]$, if $\e (\delta)$ holds, then Exp4.R with $\pmin = \sqrt{\ln{N}/(KT)}$ satisfies that, for all $i, i' \in \I$, if $\ln w_i(T+1) - \ln w_{i'}(T+1) > \epsilon_i$, then $R_i(T) > R_{i'}(T)$.
\label{lemma:exp4r_ranking}
\end{Lemma}

\begin{Cor}
Under the conditions of Lemma~\ref{lemma:exp4r_ranking}, for all $i, i' \in \I$, if $\ln w_i(T+1) - \ln w_{i'}(T+1) > \epsilon_i$, then $R_i(T) > R_{i'}(T)$; if $R_i(T) \geq R_{i'}(T)$, then $\ln w_i(T+1) - \ln w_{i'}(T+1) \geq -\epsilon_{i'}$.
\label{cor:exp4r_ranking}
\end{Cor}

\begin{Propn}
Under Assumption~\ref{assumption:regime_uniform_expert}, for any $\delta \in (0, 1]$, with probability at least $1-\delta$, Exp4.R configured with $\pmin = \sqrt{\ln{N}/(KT)}$ satisfies that
\begin{enumerate}[(i)]
	\item $\text{Regret}(T; \I) \leq 7 \sqrt{KT\ln{(2N/\delta)}}$;
	\item for all $i, i' \in \I$, if $\ln w_i(T+1) - \ln w_{i'}(T+1) > \epsilon_i$, then $R_i(T) > R_{i'}(T)$.
\end{enumerate}
\label{propn:exp4r_ranking}
\end{Propn}

\section{SELECTION AMONG INFINITELY MANY EXPERTS}\label{sec:infinite_experts}
In this section, we study the problem of nonstochastic bandits with a countably infinite set of experts. We make no assumptions about the experts or how they are indexed. For this general case, we propose a meta-algorithm called \textbf{Be}st \textbf{E}xpert \textbf{S}earch (BEES, Algorithm~\ref{alg:expert_selection_general}) that runs Exp4.R as a subroutine and provide a high-probability upper bound on regret. Section~\ref{sec:structured_experts} provides an example of structured experts and discusses how the expert ranking property of Exp4.R can be used to expedite learning in such case.

\begin{algorithm}[hb!]
	\caption{\textbf{Be}st \textbf{E}xpert \textbf{S}earch (BEES)}
	\label{alg:expert_selection_general}
\begin{algorithmic}[1]
	\State {\bfseries Input:} $\delta \in (0,1]$, $\alpha \in \Z$, $L \in \Z$, $c \in \Z$, $C \in \Z$
	\For{epoch $l = 1, \dots ,L$}
		\State $N_l \gets c2^{\alpha l}, \, T_l \gets C2^l.$
		\State $\pmin_l \gets \sqrt{\ln{N_l}/(KT_l)}.$
		\State $\I_l \gets [N_l].$
		\State Exp4.R$\left( \delta /L, T_l, \pmin_l, \I_l \right) .$
	\EndFor
\end{algorithmic}
\end{algorithm}

BEES takes five inputs including an error rate $\delta \in (0,1]$, the number of \emph{epochs} $L \in \Z$, and three constants $\alpha,c,C \in \Z$ that control the exponential growth of the epoch length and the number of experts queried in each epoch. At a high level, BEES supplies Exp4.R with an increasing (but still finite) number of experts over epochs, prioritizing those with lower indices. This scheme can be considered as putting a prior on the experts \emph{implicitly} where the experts that are believed to perform well are given low indices. The regret upper bound established in Theorem~\ref{thm:master_algo_pac_regret_general} for BEES adapts to the unknown difficulty of the problem in the sense that $i^*$ being large corresponds to a bad implicit prior. Since we make no assumptions about the experts, they can be ordered using domain knowledge before being input into BEES. Growing the epoch length and the number of experts at exponential rates allows us to derive a regret upper bound of the same order as that of Exp4.R when the best expert $i^*$ has a relatively low index. This idea is similar to, though not the same as, the doubling trick~\cite{besson2018what} as the latter only deals with the epoch length. We need to increase the number of experts at an appropriate rate relative to the epoch length.

Corollary~\ref{cor:master_algo_pac_regret_general} simplifies the bound in Theorem~\ref{thm:master_algo_pac_regret_general} for specific parameter values. Corollary~\ref{cor:master_algo_pac_regret_general} shows that BEES, when tuned right, satisfies $\text{Regret}(T) = \tbo{(i^*)^{1/\alpha}K + \sqrt{\alpha KT}}$ with high probability, where $\tbo{\cdot}$ omits only polylog factors. This upper bound illustrates the trade-off between exploration and exploitation for the problem of bandits with infinitely many experts. On the one hand, we want to include numerous experts in each epoch so as to approach $i^*$ fast. On the other hand, querying too many experts simultaneously necessitates long epochs, which reduces the rate at which more experts are included. This trade-off is controlled by $\alpha \in \Z$. The term $\tbo{(i^*)^{1/\alpha}K}$ in the bound is due to not considering $i^*$ sooner. The other term $\tbo{\sqrt{\alpha KT}}$ is the regret that benchmarks against the best expert in each epoch. Another consideration for not using an arbitrarily large value of $\alpha$ is that the minimum time horizon required by BEES which is $T = \Omega (C(\alpha, c, K, \delta))$ increases with $\alpha$. Although tuning $\alpha$ needs the unknown index $i^*$ of the best expert, we can simply set $\alpha =1$. BEES has space complexity $\bo{K(1+T/K)^\alpha}$ and time complexity $\tbo{K^2(1+T/K)^{\alpha +1}}$.

The regret bound in Theorem~\ref{thm:master_algo_pac_regret_general} matches the lower bound $\tilde{\Omega} (\sqrt{KT})$ derived by~\cite{alekh2012contextual} up to polylog factors when $i^* = \tilde{\mathcal{O}} \big( \sqrt{T/K} \big)$. This regime is less restricted than it seems at first sight. Assuming a prior distribution on $i^*$ shows that the condition on $i^*$ is satisfied with a probability that increases with $T$, the rate of which can be fast. For simplicity, let $\alpha =1$ and $c=1$. In order for $\text{Regret}(T) = \tbo{\sqrt{KT}}$ to hold with high probability, we need $i^* = \tilde{\mathcal{O}} \big( \sqrt{T/K} \big)$. We denote the complement of this event as $\B$. If we suppose that $F(i)=\pr{i^*>i}$ for $i \in \Z$ and some function $F: \Z \to [0,1]$, then $\pr{\B}$ decreases with $T$. For example, if $F(i) \propto i^{-s}$ for some $s>0$, then $\pr{\B}$ is roughly proportional to $K^{s/2}T^{-s/2}$. If $F(i) \propto e^{-si}$ for some $s>0$, then $\pr{\B}$ is roughly proportional to $e^{-s\sqrt{T/K}}$.

Although the worst-case regret guarantee of BEES is the same order as that achieved by running Exp4.P on a subset of $T$ experts for sufficiently small $i^*$, BEES can be configured to expedite learning by exploiting the expert structure if it is known, which we will discuss in Section~\ref{sec:structured_experts}. Section~\ref{sec:experiments} will show in numerical experiments that growing a subset of experts adaptively can improve performance in practice in comparison with fixing a subset of experts a priori. Moreover, the truncation method requires knowledge of $T$, which is not necessary for BEES as we can use sufficiently small $\delta$ instead of $\delta /L$ in the subroutine Exp4.R. Finally, since the computational complexity of Exp4.P is linear in the number of experts for both space and runtime, running Exp4.P on $T$ experts becomes computationally intensive for large $T$.

Before stating Theorem~\ref{thm:master_algo_pac_regret_general}, we provide some intuition for the proof. Lemma~\ref{lemma:pac_bound_on_total_reward_estimation_error} implies that $\sum_{t \in \T_l} \hat{y}_i(t) \approx R_i(\T_l)$ for each expert $i$ and every epoch $l$ with high probability. For this reason, we can prove an upper bound on the regret with respect to the best expert in each epoch, namely, $\sum_{l=1}^L R_{i^*_l}(\T_l) - \sum_{t=1}^T r_{a(t)}(t) = \tbo{\sqrt{\alpha KT}}$. We then derive an upper bound on the gap between the globally best expert and the best expert in each epoch, which is given by $R_{i^*}([T]) - \sum_{l=1}^L R_{i^*_l}(\T_l) = \tbo{(i^*)^{1/\alpha}K}$. Adding the upper bounds, we get $\text{Regret}(T) = \tbo{(i^*)^{1/\alpha}K + \sqrt{\alpha KT}}$.

For simplicity of notation, we suppose that the total number of epochs is $L=\log_2 [1+T/(2C)]$ so that $T=\sum_{l=1}^L T_l$ where $T_l = C2^l$ for $l \in [L]$. We use $\floor{\cdot}$ and $\ceil{\cdot}$ to denote the floor and ceiling functions, respectively. For the general case of $T \geq 2C$, let $L=\floor{\log_2 [1+T/(2C)]}$, $T_l = C2^l$ for $l \in [L-1]$, and $T_L = T - \sum_{l=1}^{L-1} T_l$.

\begin{Thm}
If a uniform expert is available in each epoch, then there exist absolute constants $\alpha \in \Z$ and $c \in \Z$ such that, for some $C(\alpha, c, K, \delta) \in \Z$, BEES satisfies that, for any $\delta \in (0, 1]$, with probability at least $1-\delta$, we have
\begin{align*}
	\text{Regret}(T) < &20\sqrt{\alpha K(T+2C) \ln \left(\frac{cL(2+T/C)}{\delta} \right) }\\
									&+ 2C \left( \frac{i^*}{c} \right)^{\frac{1}{\alpha}}.
\end{align*}
\label{thm:master_algo_pac_regret_general}
\end{Thm}

\begin{Cor}
Under the conditions of Theorem~\ref{thm:master_algo_pac_regret_general}, running BEES with $\alpha \in \Z$, $c \in \Z$, and $C=\ceil{\alpha K\ln (16c^4/\delta)}$ satisfies that, for any $\delta \in (0, 1]$, with probability at least $1-\delta$, $\text{Regret}(T) = \tbo{(i^*)^{1/\alpha}K + \sqrt{\alpha KT}}$.
\label{cor:master_algo_pac_regret_general}
\end{Cor}

\begin{proof}[Proof of Theorem~\ref{thm:master_algo_pac_regret_general}]
We can show that, for all $\delta \in (0, 1]$, $\alpha \in \Z$, and $c \in \Z$, there exists $C(\alpha, c, K, \delta) \in \Z$ such that $4K\ln \left( c2^{\alpha l} \right) \leq C2^l$ and $\ln \left( c2^{\alpha l+1}/\delta \right) \leq (e-2)CK2^l$ for all $l \in \Z$. For example, we can set $C=\ceil{\alpha K\ln (16c^4/\delta)}$. Together with the definitions of $N_l$ and $T_l$ in Algorithm~\ref{alg:expert_selection_general}, we have that, for all $\alpha \in \Z$ and $c \in \Z$, there exists $C \in \Z$ such that $4K\ln N_l \leq T_l$ and $\ln (2N_l/\delta) \leq (e-2)KT_l$ for all $l \in \Z$. We fix such integers $\alpha,c,C \in \Z$ for the rest of the proof.

For simplicity of notation, we first consider running Exp4.R$\left( \delta, T_l, \pmin_l, \I_l \right)$ in each epoch $l$ for any $\delta \in (0, 1/L]$ and then apply a change of variables at the end of the proof. We suppose that a uniform expert is available in each epoch. Assumption~\ref{assumption:regime_uniform_expert} is then satisfied for all epochs. For now, we assume that event $\e (\delta)$ holds for all epochs, the probability of which will be discussed at the end of the proof. For simplicity of notation, let $i^*_l \defeq i^*(\I_l;\T_l)$ for $l \in [L]$.

Let $U_l \defeq \alpha l+\log_2{(2c/\delta)}$ for $l \in [L]$. Recall that $\T_l$ is the time interval of epoch $l$ where $|\T_l|=T_l$. By Lemma~\ref{lemma:beygelzimer11_exp4p_regret},
\begin{align*}
	\sum_{l=1}^L R_{i^*_l}(\T_l) - \sum_{t=1}^T r_{a(t)}(t)
	&\leq \sum_{l=1}^L 7\sqrt{KT_l \ln \left( \frac{2N_l}{\delta} \right)} \\
	&= 7\sqrt{KC\ln{2}} \sum_{l=1}^L \sqrt{2^lU_l} \\
	&\leq 7\sqrt{KCU_L\ln{2}} \sum_{l=1}^L 2^{l/2} \\
	&< 20\sqrt{KCU_L} \left( 2^{L/2} -1 \right) .
\end{align*}
Since $L=\log_2 [1+T/(2C)]$, we have
\begin{equation}
	\begin{aligned}
		&\quad \, \sum_{l=1}^L R_{i^*_l}(\T_l) - \sum_{t=1}^T r_{a(t)}(t) \\
		&< 20\sqrt{KCU_L} \left( \sqrt{1+\frac{T}{2C}} -1 \right) \\
		&< 20\sqrt{K\left[ \alpha L+2\ln \left( \frac{2c}{\delta} \right) \right] \left( C+\frac{T}{2} \right) }.
	\end{aligned}
\label{eq:master_algo_proof_estimation_error_upper_bound_general}
\end{equation}

We first discuss the case where $i^* \notin \I_1$. Let $L'$ be the last epoch such that $i^*$ is not considered in Algorithm~\ref{alg:expert_selection_general}. Since $| \I_l | = N_l$, we have $L' = \min \left( L, \ceil{\alpha^{-1}\log_2(i^*/c)} -1 \right)$. Since $i^* \notin \I_1$, we get $L' \geq 1$. By the definition of $i^*_l$, we have $R_{i^*_l}(\T_l) \geq R_{i^*}(\T_l)$ for all $l>L'$. Thus,
\begin{equation}
	\begin{aligned}
		R_{i^*}([T]) - \sum_{l=1}^L R_{i^*_l}(\T_l)
		&\leq \sum_{l=1}^{L'} \left( R_{i^*}(\T_l) - R_{i^*_l}(\T_l) \right) \\
		&\leq \sum_{l=1}^{L'} T_l \\
		&< C2^{L'+1} \\
		&< 2C \left( \frac{i^*}{c} \right)^{\frac{1}{\alpha}}.
	\end{aligned}
\label{eq:master_algo_proof_approximation_error_upper_bound_general}
\end{equation}

We now consider the case where $i^* \in \I_1$. It follows from Algorithm~\ref{alg:expert_selection_general} that $i^* \in \I_l$ for all $l$. Thus, the definition of $i^*_l$ implies that $R_{i^*_l}(\T_l) \geq R_{i^*}(\T_l)$ for all $l$. We define $D \defeq R_{i^*}([T]) - \sum_{l=1}^L R_{i^*_l}(\T_l)$. We then have $D \leq 0$. However, the definition of $i^*$ implies that $D \geq 0$. Therefore, $D=0$ and \eqref{eq:master_algo_proof_approximation_error_upper_bound_general} is satisfied.

Adding \eqref{eq:master_algo_proof_estimation_error_upper_bound_general} and \eqref{eq:master_algo_proof_approximation_error_upper_bound_general} gives
\begin{equation}
\begin{aligned}
	\text{Regret}(T) < &20\sqrt{K\left[ \alpha L+2\ln \left( \frac{2c}{\delta} \right) \right] \left( C+\frac{T}{2} \right) } \\
									&+ 2C \left( \frac{i^*}{c} \right)^{\frac{1}{\alpha}}.
\label{eq:master_algo_proof_regret_before_change_of_variables_general}
\end{aligned}
\end{equation}
Using Lemma~\ref{lemma:pac_bound_on_total_reward_estimation_error} and the union bound over all $L$ epochs, we conclude that \eqref{eq:master_algo_proof_regret_before_change_of_variables_general} holds with probability at least $1-L\delta$. A change of variables gives that, for any $\delta \in (0, 1]$, with probability at least $1-\delta$, we have
\begin{align*}
	\text{Regret}(T)
	< &20\sqrt{K\left[ \alpha L+2\ln \left( \frac{2cL}{\delta} \right) \right] \left( C+\frac{T}{2} \right) } \\
		&+ 2C \left( \frac{i^*}{c} \right)^{\frac{1}{\alpha}} \\
	< &20\sqrt{\alpha K(T+2C) \ln \left(\frac{cL(2+T/C)}{\delta} \right) } \\
		&+ 2C \left( \frac{i^*}{c} \right)^{\frac{1}{\alpha}}.
\end{align*}
\end{proof}

\subsection{Structured Experts}\label{sec:structured_experts}
In this section, we present an example of structured experts that is inspired by the problem of finite-time model selection for RL and discuss how the expert ranking property of Exp4.R can be used to expedite learning in such case.

As RL becomes increasingly integrated into autonomous systems such as agile robots~\cite{hwangbo2019learning}, self-driving vehicles~\cite{kuderer2015learning}, customized fertilizer formulation~\cite{binas2019reinforcement}, and personalized medication dosing~\cite{nemati2016optimal}, it is crucial that the techniques are robust~\cite{matni2019from}. An aspect of robustness is the capability to detect and adjust for model errors. For RL, this entails both model selection and parameter estimation. How to achieve both objectives simultaneously while maintaining provably good performance is an active area of research~\cite{ni2019maximum, abbasi2020regret}. The crux of the problem of online model selection for RL is to balance approximation and estimation errors in a time-dependent manner. As an example, we suppose that there is an infinite sequence of nested model classes. This structure arises naturally when an RL algorithm incorporates increasingly many features over time. Some new features may also just become obtainable while an RL algorithm is running. In fact, it is unknown a priori for many applications what is a minimal feature space that contains an optimal policy. Given an infinite sequence of model classes, the best class to use depends on the horizon or, equivalently, the amount of trajectory data that will become available. Although a larger model class has a smaller approximation error, it tends to have a higher estimation error for a fixed finite horizon. Moreover, if several classes have the same approximation power, the simplest one is typically preferred in consideration of time and space complexity.

Inspired by the problem of finite-time model selection for RL, we propose to consider experts structured in a way that simulates the trade-off between approximation and estimation. In particular, we suppose that the experts are ranked in ascending order of complexity. We propose a variant of BEES which also operates in $L=\bo{\ln{T}}$ epochs with $\T_l$ being the time interval of epoch $l$. Assumption~\ref{aspn:time_variant_model_ranking} stipulates that the total reward is weakly unimodal in expectation with respect to the expert index during any epoch. In addition, the index of the globally best expert is nondecreasing as the epoch increases. See Fig.~\ref{fig:time_variant_assumption} for an illustration. Section~\ref{sec:experiments} will demonstrate in numerical experiments that a noisy unimodal structure can be sufficient in practice.

\begin{figure}[t!]
	\centering
	\includegraphics[width=0.4\textwidth]{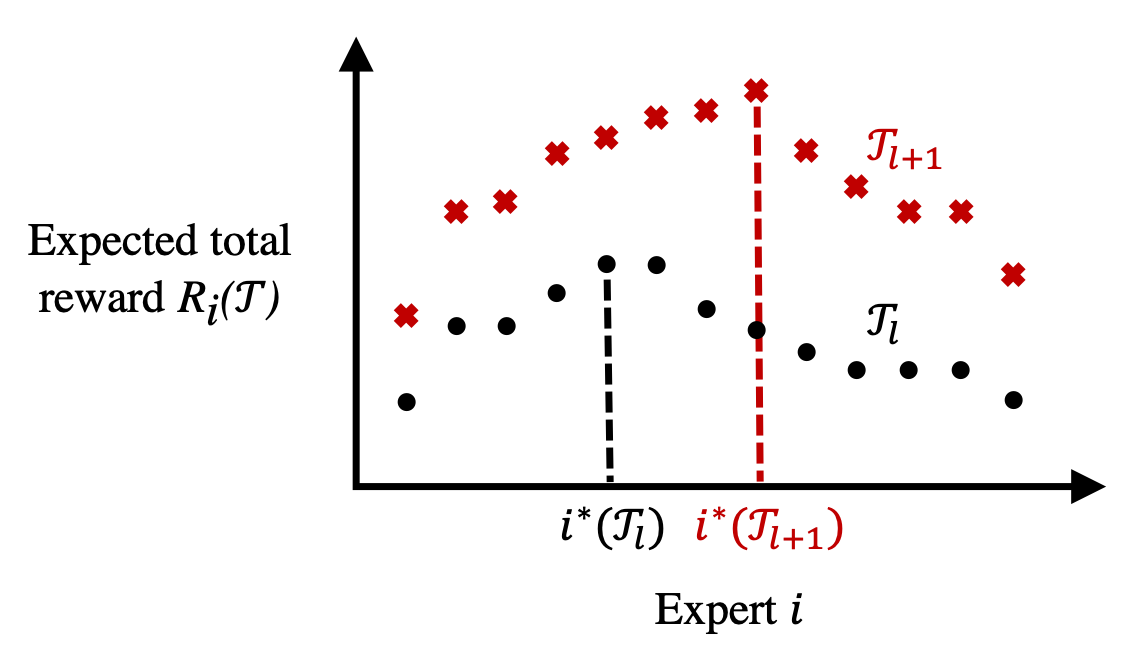}
	\caption{An illustration of Assumption~\ref{aspn:time_variant_model_ranking}.}
\label{fig:time_variant_assumption}
\end{figure}

\begin{Aspn}
For any epoch $l \in [L]$, if $i \leq i^*(\Z;\T_l)$, then $R_{i-1}(\T_l) \leq R_i(\T_l)$. Otherwise, $R_i(\T_l) \geq R_{i+1}(\T_l)$. Moreover, $i^*(\Z;\T_l) \leq i^*(\Z;\T_{l'})$ if $l < l'$.
\label{aspn:time_variant_model_ranking}
\end{Aspn}

The proposed time-dependent unimodal structure is fundamentally related to oracle inequalities in empirical risk minimization~\cite{wainwright2019high}. Although the experts' performance may fluctuate around the proposed structure in practice, solutions to the stylized setting are of theoretical interest. Unimodal bandits have been previously studied for the stochastic setting where the expected reward is a unimodal function of partially ordered arms~\cite{cope2009regret, yu2011unimodal, combes2014unimodal, combes2020unimodal}. Extensions to non-stationary environments have been proposed for low-frequency abrupt changes~\cite{yu2011unimodal} and smooth changes~\cite{combes2014unimodal} in expected rewards. Our setting is a nonstochastic bandit problem with no assumptions on the frequency or the magnitude of changes in the unimodal structure.

Under Assumption~\ref{aspn:time_variant_model_ranking}, the outputs of Exp4.R give a threshold rule that allows us to find a lower bound for $i^*$, which can accelerate the rate of approaching $i^*$. We modify BEES to incorporate lower bound estimation (BEES.LB, Algorithm~\ref{alg:expert_selection_improved}). BEES.LB runs Exp4.R and Probabilistic Thresholding Search (PTS, Algorithm~\ref{alg:search_improved}) as subroutines. In each epoch, BEES.LB eliminates experts identified as suboptimal. Lemma~\ref{lemma:lower_bound_on_optimal_model} shows that the estimated lower bound is correct if the concentration event $\e (\delta)$ holds. Theorem~\ref{thm:master_algo_pac_regret_improved} establishes a high-probability regret upper bound for BEES.LB. The proof is similar to that of Theorem~\ref{thm:master_algo_pac_regret_general}, hence provided in the appendix. PTS has space complexity $\bo{N}$ and time complexity $\bo{N^2}$. PTS can be efficiently implemented by first sorting the input $w$. BEES.LB takes the same space $\bo{K(1+T/K)^\alpha}$ as BEES. The time complexity of BEES.LB is $\tbo{K^2(1+T/K)^{\alpha +1} + (1+T/K)^{2\alpha}}$, which reduces to the runtime of BEES for sufficiently small $\alpha$.

\begin{algorithm}[th!]
	\caption{\textbf{BEES} with \textbf{L}ower \textbf{B}ound (BEES.LB)}
	\label{alg:expert_selection_improved}
\begin{algorithmic}[1]
	\State {\bfseries Input:} $\delta \in (0,1]$, $\alpha \in \Z$, $L \in \Z$, $c \in \Z$, $C \in \Z$
	\State $\underline{i}_1 \gets 1.$
	\For{epoch $l = 1, \dots ,L$}
		\State $N_l \gets c2^{\alpha l}, \, T_l \gets C2^l.$
		\State $\pmin_l \gets \sqrt{\ln{N_l}/(KT_l)}.$
		\State $\I_l \gets \{ \underline{i}_l, \underline{i}_l+1, \dots , \underline{i}_l + N_l-1\}.$
		\State $w^l, \epsilon^l \gets$ Exp4.R$\left( \delta /L, T_l, \pmin_l, \I_l \right) .$
		\State $\underline{i}_{l+1} \gets \text{PTS} \left( w^l, \epsilon^l, \underline{i}_l \right) .$
	\EndFor
\end{algorithmic}
\end{algorithm}

\begin{algorithm}[h!]
	\caption{Probabilistic Thresholding Search (PTS)}
	\label{alg:search_improved}
\begin{algorithmic}
	\State {\bfseries Input:} $w \in (0, \infty)^N$, $\epsilon \in (0, \infty)^N$, $\underline{i} \in \Z$
	\State {\bfseries Output:} $\underline{i}_{\text{new}}$
	\State $\underline{j} \gets 1.$
	\For{$j = 1, \dots , N-1$}
		\For{$j' = j+1, \dots ,N$}
			\If{$\ln{w_{j'}} - \ln{w_j} > \epsilon_{j'}$}
				\State $\underline{j} \gets j+1.$
			\EndIf
		\EndFor
	\EndFor
	\State $\underline{i}_{\text{new}} \gets \underline{i} + \underline{j}-1.$
\end{algorithmic}
\end{algorithm}

\begin{Lemma}
Under Assumption~\ref{aspn:time_variant_model_ranking} and the conditions of Lemma~\ref{lemma:exp4r_ranking}, if event $\e (\delta)$ holds for all epochs, then $\underline{i}_l \leq i^*$ for all $l$.
\label{lemma:lower_bound_on_optimal_model}
\end{Lemma}

\begin{Thm}
Under Assumption~\ref{aspn:time_variant_model_ranking}, if a uniform expert is available in each epoch, then there exist absolute constants $\alpha \in \Z$ and $c \in \Z$ such that, for some $C(\alpha, c, K, \delta) \in \Z$, BEES.LB satisfies that, for any $\delta \in (0, 1]$, with probability at least $1-\delta$, we have
\begin{align*}
	\text{Regret}(T) < &20\sqrt{\alpha K(T+2C) \ln \left(\frac{cL(2+T/C)}{\delta} \right) }\\
									&+ 2C \left( \frac{i^*}{c} \right)^{\frac{1}{\alpha}}.
\end{align*}
\label{thm:master_algo_pac_regret_improved}
\end{Thm}

The upper bound in Theorem~\ref{thm:master_algo_pac_regret_improved} is the same as that for the general case of unstructured experts because the lower bound from PTS can stay at $1$ in the worst case. A trivial example is that all experts are the same. For cases where the experts' performance differs by sufficient margins, the actual improvement of BEES.LB over BEES should become obvious as we will demonstrate in Section~\ref{sec:experiments}.

If the globally best expert $i^*$ is fixed over time, then we can modify BEES.LB to additionally estimate an upper bound on $i^*$, initialized to $\infty$. The modified search subroutine can be considered as a probabilistic counterpart of search algorithms such as the golden-section search~\cite{kiefer1953sequential}. The major difference is that the search subroutine applies to problems where the function cannot be evaluated directly. We can show that the confidence interval for $i^*$ contracts over epochs. While the epoch length always grows exponentially, the set of experts considered in each epoch is data-dependent. If no upper bound on $i^*$ has been identified, then the number of experts considered will increase by a factor of $2^{\alpha}$ in the next epoch. Otherwise, only the experts in the non-expanding confidence interval will be considered from now on.

\section{EXPERIMENTS}\label{sec:experiments}
We conduct numerical simulations to demonstrate the performance improvement of BEES.LB in comparison with BEES and Exp4.P when experts are structured. We consider $K=10$ actions the rewards of which are binary and nonstochastic. The sequence of experts has a weakly unimodal structure that is corrupted with random noise. The first expert is uniform and the best expert has index $i^*=9$. At each time, every expert's advice is distorted with an additive $K$-vector that consists of independent zero-mean Gaussian noises with standard deviation $0.01$, which may alter the unimodal structure of the experts. For BEES and BEES.LB, we set $\alpha =c=1$ and $C$ as defined in Corollary~\ref{cor:master_algo_pac_regret_general}. We implement the version of BEES and BEES.LB that does not know the number of epochs $L$ or the time horizon $T$ a priori by using $\delta$ instead of $\delta /L$ in the subroutine Exp4.R. In contrast, we configure the benchmark algorithm Exp4.P with the correct $T$. Exp4.P is run on the first $T$ experts in the sequence. All algorithms use an error rate $\delta =0.05$.

Fig.~\ref{fig:regret} shows that BEES.LB indeed has a lower regret than BEES because of the expedited learning enabled by Exp4.R. For large $T$, Exp4.P is surpassed by BEES.LB as querying too many experts can increase the chance of getting bad advice. Fig.~\ref{fig:regret} demonstrates the advantage of our algorithm in having improved performance by being able to exploit structural information and query experts adaptively.

\begin{figure}[b!]
	\centering
	\includegraphics[scale=0.5]{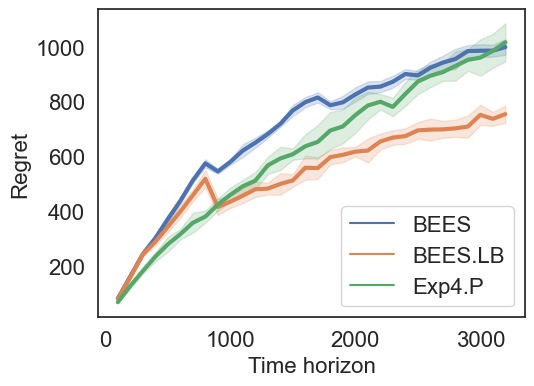}
	\caption{Comparison of BEES, BEES.LB, and Exp4.P in terms of regret as the time horizon varies. BEES.LB surpasses BEES and Exp4.P as the time horizon increases. Lines and shades are the averages and the standard deviations of $10$ runs, respectively.}
\label{fig:regret}
\end{figure}

\section{DISCUSSION}\label{sec:conclusions}
In this paper, we have proposed an algorithm for the problem of nonstochastic bandits with infinitely many experts under the constraint of having access to only a finite subset of experts. We have established a high-probability upper bound on the regret of our meta-algorithm BEES, which matches the lower bound up to polylog factors if the globally best expert has a relatively low index. If we assume that there exists a prior distribution on the best expert, then the probability that our regret upper bound is tight will increase with the time horizon, the rate of which can be fast. The expert ranking property of the subroutine Exp4.R enables learning acceleration if the structure of the experts is known. We have illustrated this point with an example that is inspired by the problem of finite-time model selection for RL. One interesting direction for future work is to obtain instance-dependent upper bounds in terms of the experts' suboptimality gaps. Such instance-dependent bounds can be used to prove the learning acceleration enabled by Exp4.R. It is also worthwhile to design efficient implementation for specific applications.

\section{ACKNOWLEDGMENTS}

The authors would like to thank John N. Tsitsiklis, Dylan J. Foster, Caroline Uhler, Devavrat Shah, and Thibaut Horel for insightful comments that helped to improve the paper.

\appendix

\section{Proof of Lemma~\ref{lemma:pac_bound_on_total_reward_estimation_error}}
Let $\mathbb{E}_t \sbr{\cdot}$ denote the conditional expectation given history until time $t-1$. We can show that $\hat{y}_i(t)$ is a conditionally unbiased estimator for $y_i(t)$. In other words, $\mathbb{E}_t \sbr{\hat{y}_i(t)} = y_i(t)$ for all $i$ and $t$. Lemma~\ref{lemma:beygelzimer11_upper_bound_on_error_variance} shows that $\hat{v}_i(t)$ is an upper bound on the conditional variance of $\hat{y}_i(t)$. Lemma~\ref{thm:beygelzimer11_martingale} is a Freedman-style inequality for martingales from \cite{beygelzimer2011contextual}. The proof of Lemma~\ref{lemma:pac_bound_on_total_reward_estimation_error} relies on Lemmas~\ref{lemma:beygelzimer11_upper_bound_on_error_variance} and \ref{thm:beygelzimer11_martingale}.

\begin{Lemma}[From proof of Lemma 3 in~\cite{beygelzimer2011contextual}]
For all $t \in \Z$ and $i \in \I$, we have $\mathbb{E}_t \sbr{\left( y_i(t) - \hat{y}_i(t) \right)^2} \leq \hat{v}_i(t)$.
\label{lemma:beygelzimer11_upper_bound_on_error_variance}
\end{Lemma}

\begin{Lemma}[\cite{beygelzimer2011contextual}, Theorem 1]
Let $X_1, \dots , X_T$ be a sequence of real-valued random variables. For any real-valued random variable $Y$, we define $\mathbb{E}_t \sbr{Y} \defeq \cE{Y}{X_1, \dots , X_{t-1}}$. We assume that, $X_t \leq B$ and $\mathbb{E}_t \sbr{X_t} =0$ for all $t$. We define the random variables
\begin{equation*}
	S \defeq \sum_{t=1}^T X_t, \quad V \defeq \sum_{t=1}^T \mathbb{E}_t \sbr{X_t^2}.
\end{equation*}
For any fixed estimate $V'>0$ of $V$, and for any $\delta \in (0, 1]$, with probability at least $1-\delta$, we have
\begin{equation*}
	\renewcommand{\arraystretch}{1.2}
	S \leq \left\{
		\begin{array}{@{}l@{\ \ }l@{}}
			\sqrt{(e-2)\ln \left( \frac{1}{\delta} \right) } \left( \frac{V}{\sqrt{V'}} + \sqrt{V'} \right) , & \text{if } V' \geq \frac{B^2\ln (1/\delta)}{e-2}, \\
			B \ln (1/\delta) + (e-2) \frac{V}{B}, & \text{otherwise}.
		\end{array}
	\right.
\end{equation*}
\label{thm:beygelzimer11_martingale}
\end{Lemma}

\begin{proof}[Proof of Lemma~\ref{lemma:pac_bound_on_total_reward_estimation_error}]
We now fix any $i \in \I$ and $t \in \Z$. By definition, we have $y_i(t) \in [0,1]$. Using~\eqref{eq:exp4.r_action_prob} and the assumption that $\pmin \in [0, 1/K]$, we get $p_a(t) \geq \pmin$ for all $a \in \A$. Thus, \eqref{eq:exp4.r_update_yv} implies that $\hat{y}_i(t) \in [0, 1/ \pmin ]$ almost surely. Let ${X_t = y_i(t) - \hat{y}_i(t)}$. We then have $-1/\pmin \leq X_t \leq 1$ almost surely. We can show that $\mathbb{E}_t \sbr{\hat{y}_i(t)} = y_i(t)$ and hence $\mathbb{E}_t \sbr{X_t} = 0$. We recall that $R_i(T) = \sum_{t=1}^T y_i(t)$. Applying Lemma~\ref{thm:beygelzimer11_martingale} to $(X_t)_t$ and $(-X_t)_t$ respectively and then taking a union bound, we conclude that, for any $\delta \in (0, 1]$, with probability at least $1-\delta / N$, the inequality ${-B_1 \leq R_i(T) - \hat{R}_i(T) \leq B_2}$ holds, where
\begin{align*}
	\renewcommand{\arraystretch}{1.2}
	B_1 \! &\defeq \! \left\{
		\begin{array}{@{}l@{\ }l@{}}
			\sqrt{(e-2)\ln \left( \frac{2N}{\delta} \right)} \left( \frac{V}{\sqrt{V'}} + \sqrt{V'} \right) \! , & \text{if } V' \geq \frac{\ln (2N/\delta)}{(e-2)\pmin^2}, \\
			\frac{\ln (2N/\delta)}{\pmin} + (e-2) \pmin V, & \text{otherwise},
		\end{array}
	\right. \\
	B_2 \! &\defeq \! \left\{
		\begin{array}{@{}l@{\ }l@{}}
			\sqrt{(e-2)\ln \left( \frac{2N}{\delta} \right)} \left( \frac{V}{\sqrt{V'}} + \sqrt{V'} \right) \! , & \text{if } V' \geq \frac{\ln (2N/\delta)}{(e-2)}, \\
			\ln (2N/\delta) + (e-2)V, & \text{otherwise},
		\end{array}
	\right. \\
	V \! &\defeq \! \sum_{t=1}^T \mathbb{E}_t \sbr{X_t^2}.
\end{align*}

We now fix an arbitrary $\delta \in (0, 1]$. Assumption~\ref{assumption:regime_uniform_expert} implies that $\ln (2N/\delta) \leq (e-2)KT$. Taking $\pmin = \sqrt{\ln{N}/(KT)}$ and $V' = KT$, we have
\begin{equation*}
	\frac{\ln (2N/\delta)}{e-2} \leq V' < \frac{\ln (2N/\delta)}{(e-2) \pmin^2}.
\end{equation*}
Lemma~\ref{lemma:beygelzimer11_upper_bound_on_error_variance} implies that $V \leq \hat{V}_i(T)$. Therefore, with probability at least $1-\delta / N$, we have
\begin{align*}
	&\quad -\ln \left( \frac{2N}{\delta} \right) \sqrt{\frac{KT}{\ln N}} - \sqrt{\frac{\ln N}{KT}} \hat{V}_i(T) \\
	&\leq R_i(T) - \hat{R}_i(T) \\
	&\leq \sqrt{\ln \left( \frac{2N}{\delta} \right)} \left( \frac{\hat{V}_i(T)}{\sqrt{KT}} + \sqrt{KT} \right) .
\end{align*}
Applying the union bound over $i \in \I$, we conclude that $\pr{\e (\delta)} \geq 1-\delta$.
\end{proof}

\begin{proof}[Proof of Lemma~\ref{lemma:exp4r_ranking}]
We fix an arbitrary $\delta \in (0, 1]$ and suppose that event $\e (\delta)$ holds. We recall that
\begin{equation*}
	\epsilon_i \defeq \left[ 1 + \frac{\hat{V}_i(T)}{KT} \right] \ln \left( \frac{2N}{\delta} \right) .
\end{equation*}
We assume that $\ln w_i(T+1) - \ln w_{i'}(T+1) > \epsilon_i$ for some $i, i' \in \I$. By \eqref{eq:exp4.r_update_w} and the initialization condition $w_i(1)=1$, we have
\begin{align*}
	\ln w_i(T+1) &= \sum_{t=1}^T \ln \left( \frac{w_i(t+1)}{w_i(t)} \right) \\
	&= \frac{\pmin}{2} \left( \hat{R}_i(T) + \sqrt{\frac{\ln (2N/\delta)}{KT}} \hat{V}_i(T) \right) .
\end{align*}
Thus,
\begin{equation}
	\hat{R}_i(T) = \frac{2}{\pmin} \ln w_i(T+1) - \sqrt{\frac{\ln (2N/\delta)}{KT}} \hat{V}_i(T).
\label{eq:proof_exp4r_ranking_total_reward_estimate}
\end{equation}
Equation~\eqref{eq:proof_exp4r_ranking_total_reward_estimate} also holds for $i'$. Thus,

\begin{equation}
	\begin{aligned}
		&\quad\ \hat{R}_i(T) - \hat{R}_{i'}(T) \\
		&= \frac{2}{\pmin} \ln \left( \frac{w_i(T+1)}{w_{i'}(T+1)} \right) \\
		&\ \ \ - \sqrt{\frac{\ln (2N/\delta)}{KT}} \left( \hat{V}_i(T) - \hat{V}_{i'}(T) \right) \\
		&> \frac{2\epsilon_i}{\pmin} - \sqrt{\frac{\ln (2N/\delta)}{KT}} \left( \hat{V}_i(T) - \hat{V}_{i'}(T) \right) \\
		&= 2\ln \left( \frac{2N}{\delta} \right) \sqrt{\frac{KT}{\ln N}} + \sqrt{\frac{\ln (2N/\delta)}{KT}} \hat{V}_{i'}(T) \\
		&\ \ \ + \hat{V}_i(T) \sqrt{\frac{\ln (2N/\delta)}{KT}} \left[ 2\sqrt{\frac{\ln (2N/\delta)}{\ln N}} - 1 \right] \\
		&> 2\ln \left( \frac{2N}{\delta} \right) \sqrt{\frac{KT}{\ln N}} \\
		&\ \ \ + \sqrt{\frac{\ln (2N/\delta)}{KT}} \left( \hat{V}_i(T) + \hat{V}_{i'}(T) \right) .
	\end{aligned}
\label{eq:proof_exp4r_ranking_total_reward_estimate_diff}
\end{equation}

Event $\e (\delta)$ implies that
\begin{equation}
	\begin{aligned}
		&\quad\ R_i(T) - \hat{R}_i(T) + \hat{R}_{i'}(T) - R_{i'}(T) \\
		&\geq -\ln \left( \frac{2N}{\delta} \right) \sqrt{\frac{KT}{\ln N}} - \sqrt{\frac{\ln N}{KT}} \hat{V}_i(T) \\
		&\ \ \ - \sqrt{\ln \left( \frac{2N}{\delta} \right)} \left( \frac{\hat{V}_{i'}(T)}{\sqrt{KT}} + \sqrt{KT} \right) .
	\end{aligned}
\label{eq:proof_exp4r_ranking_total_reward_estimate_error_sum}
\end{equation}

Adding \eqref{eq:proof_exp4r_ranking_total_reward_estimate_diff} and \eqref{eq:proof_exp4r_ranking_total_reward_estimate_error_sum} and then simplifying the algebra give
\begin{equation*}
	R_i(T) - R_{i'}(T) >0.
\end{equation*}
\end{proof}

\begin{proof}[Proof of Proposition~\ref{propn:exp4r_ranking}]
Proposition~\ref{propn:exp4r_ranking} follows directly from Lemmas~\ref{lemma:pac_bound_on_total_reward_estimation_error}--\ref{lemma:exp4r_ranking}.
\end{proof}

\section{Proof of Theorem~\ref{thm:master_algo_pac_regret_improved}}

\begin{proof}[Proof of Lemma~\ref{lemma:lower_bound_on_optimal_model}]
Under the assumption that event $\e (\delta)$ holds for all epochs, we prove the statement by induction on $l$. The base case holds trivially as $\underline{i}_1 =1$. For the inductive step, we assume that $\underline{i}_{\iota} \leq i^*$ for all $\iota \leq l$. If $\underline{i}_{l+1} = \underline{i}_l$, then $i^* \geq \underline{i}_{l+1}$ by the induction hypothesis. If there exists some $j \geq 1$ such that $\underline{i}_{l+1} = \underline{i}_l +j$, then Algorithm~\ref{alg:search_improved} implies that $\ln{w_{j'}} - \ln{w_j} > \epsilon_{j'}$ for some $j'>j$ in epoch $l$. Using Assumption~\ref{aspn:time_variant_model_ranking} and Lemma~\ref{lemma:exp4r_ranking}, we get $i^* \geq \underline{i}_l+j = \underline{i}_{l+1}$.
\end{proof}

\begin{proof}[Proof of Theorem~\ref{thm:master_algo_pac_regret_improved}]
We can show that, for all $\delta \in (0, 1]$, $\alpha \in \Z$, and $c \in \Z$, there exists $C(\alpha, c, K, \delta) \in \Z$ such that $4K\ln \left( c2^{\alpha l} \right) \leq C2^l$ and $\ln \left( c2^{\alpha l+1}/\delta \right) \leq (e-2)CK2^l$ for all $l \in \Z$. For example, we can set $C=\ceil{\alpha K\ln (16c^4/\delta)}$. Together with the definitions of $N_l$ and $T_l$ in Algorithm~\ref{alg:expert_selection_improved}, we have that, for all $\alpha \in \Z$ and $c \in \Z$, there exists $C \in \Z$ such that $4K\ln N_l \leq T_l$ and $\ln (2N_l/\delta) \leq (e-2)KT_l$ for all $l \in \Z$. We fix such integers $\alpha,c,C \in \Z$ for the rest of the proof.

For simplicity of notation, we first consider running Exp4.R$\left( \delta, T_l, \pmin_l, \I_l \right)$ in each epoch $l$ of Algorithm~\ref{alg:expert_selection_improved} for any $\delta \in (0, 1/L]$ and then apply a change of variables at the end of the proof. We suppose that a uniform expert is available in each epoch. Assumption~\ref{assumption:regime_uniform_expert} is then satisfied for all epochs. For now, we assume that event $\e (\delta)$ holds for all epochs, the probability of which will be discussed at the end of the proof. For simplicity of notation, let $i^*_l \defeq i^*(\I_l;\T_l)$ for $l \in [L]$.

Let $U_l \defeq \alpha l+\log_2{(2c/\delta)}$ for $l \in [L]$. Recall that $\T_l$ is the time interval of epoch $l$ where $|\T_l|=T_l$. By Lemma~\ref{lemma:beygelzimer11_exp4p_regret},
\begin{align*}
	\sum_{l=1}^L R_{i^*_l}(\T_l) - \sum_{t=1}^T r_{a(t)}(t)
	&\leq \sum_{l=1}^L 7\sqrt{KT_l \ln \left( \frac{2N_l}{\delta} \right)} \\
	&= \sum_{l=1}^L 7\sqrt{KC2^l \ln \left( \frac{c2^{\alpha l+1}}{\delta} \right)} \\
	&= 7\sqrt{KC\ln{2}} \sum_{l=1}^L \sqrt{2^lU_l} \\
	&\leq 7\sqrt{KCU_L\ln{2}} \sum_{l=1}^L 2^{l/2} \\
	&< 20\sqrt{KCU_L} \left( 2^{L/2} -1 \right) .
\end{align*}
Since $L=\log_2 [1+T/(2C)]$, we have
\begin{equation}
	\begin{aligned}
		&\quad \, \sum_{l=1}^L R_{i^*_l}(\T_l) - \sum_{t=1}^T r_{a(t)}(t) \\
		&< 20\sqrt{KCU_L} \left( \sqrt{1+\frac{T}{2C}} -1 \right) \\
		&< 20\sqrt{K\left[ \alpha L+2\ln \left( \frac{2c}{\delta} \right) \right] \left( C+\frac{T}{2} \right) }.
	\end{aligned}
\label{eq:master_algo_proof_estimation_error_upper_bound_improved}
\end{equation}

We first discuss the case where $i^* \notin \I_1$. Let $L''$ be the last epoch such that $i^*$ is not considered in Algorithm~\ref{alg:expert_selection_improved}. In other words, $L'' \defeq \max \sth{l \in [L]}{i^* \notin \I_l}$. Lemma~\ref{lemma:lower_bound_on_optimal_model} implies that $i^* \in \I_l$ for all $l>L''$. By the definition of $i^*_l$, we have $R_{i^*_l}(\T_l) \geq R_{i^*}(\T_l)$ for all $l>L''$. Thus,
\begin{align*}
	R_{i^*}([T]) - \sum_{l=1}^L R_{i^*_l}(\T_l)
	&\leq \sum_{l=1}^{L''} \left( R_{i^*}(\T_l) - R_{i^*_l}(\T_l) \right) \\
	&\leq \sum_{l=1}^{L''} T_l \\
	&< C2^{L''+1}.
\end{align*}
We now provide an upper bound on $L''$. By Algorithms~\ref{alg:expert_selection_improved} and \ref{alg:search_improved}, we have $| \I_l | = N_l$ and $1 \leq \underline{i}_l \leq \underline{i}_{l+1}$ for all $l$. Let $L'$ be the last epoch such that $i^*$ is not considered in the worst case where $\underline{i}_l = 1$ for all $l$. In other words, $L' \defeq \min \left( L, \ceil{\alpha^{-1}\log_2(i^*/c)} -1 \right)$. Under the assumption that $i^* \notin \I_1$, we get $L' \geq 1$. By the definitions of $L'$ and $L''$, we have $L'' \leq L'$ and hence
\begin{equation}
	R_{i^*}([T]) - \sum_{l=1}^L R_{i^*_l}(\T_l) < C2^{L'+1} < 2C \left( \frac{i^*}{c} \right)^{\frac{1}{\alpha}}.
\label{eq:master_algo_proof_approximation_error_upper_bound_improved}
\end{equation}

We now consider the case where $i^* \in \I_1$. It follows from Lemma~\ref{lemma:lower_bound_on_optimal_model} that $i^* \in \I_l$ for all $l$. Thus, the definition of $i^*_l$ implies that $R_{i^*_l}(\T_l) \geq R_{i^*}(\T_l)$ for all $l$. We define $D \defeq R_{i^*}([T]) - \sum_{l=1}^L R_{i^*_l}(\T_l)$. We then have $D \leq 0$. However, the definition of $i^*$ implies that $D \geq 0$. Therefore, $D=0$ and \eqref{eq:master_algo_proof_approximation_error_upper_bound_improved} is satisfied.

Adding \eqref{eq:master_algo_proof_estimation_error_upper_bound_improved} and \eqref{eq:master_algo_proof_approximation_error_upper_bound_improved} gives
\begin{equation}
\begin{aligned}
	\text{Regret}(T) < &20\sqrt{K\left[ \alpha L+2\ln \left( \frac{2c}{\delta} \right) \right] \left( C+\frac{T}{2} \right) } \\
									&+ 2C \left( \frac{i^*}{c} \right)^{\frac{1}{\alpha}}.
\label{eq:master_algo_proof_regret_before_change_of_variables_improved}
\end{aligned}
\end{equation}
Using Lemma~\ref{lemma:pac_bound_on_total_reward_estimation_error} and the union bound over all $L$ epochs, we conclude that \eqref{eq:master_algo_proof_regret_before_change_of_variables_improved} holds with probability at least $1-L\delta$. A change of variables gives that, for any $\delta \in (0, 1]$, with probability at least $1-\delta$, we have
\begin{align*}
	\text{Regret}(T)
	< &20\sqrt{K\left[ \alpha L+2\ln \left( \frac{2cL}{\delta} \right) \right] \left( C+\frac{T}{2} \right) } \\
		&+ 2C \left( \frac{i^*}{c} \right)^{\frac{1}{\alpha}} \\
	< &20\sqrt{\alpha K(T+2C) \ln \left(\frac{cL(2+T/C)}{\delta} \right) } \\
		&+ 2C \left( \frac{i^*}{c} \right)^{\frac{1}{\alpha}}.
\end{align*}
\end{proof}


\bibliographystyle{ieeetr}
\bibliography{InfiniteExperts}

\end{document}